\documentclass{article}
\usepackage{amsmath, amsthm, amssymb}
\usepackage{algorithm,algorithmic}
\usepackage{fullpage}
\usepackage{graphicx}
\usepackage{caption}
\usepackage{natbib}
\newcommand{\vol}[1]{\ensuremath{V^{#1}}}
\newcommand{\alloc}[2]{\ensuremath{v_{#1}^{#2}}}
\newcommand{\alloci}[2]{\ensuremath{u_{#1}^{#2}}}
\newcommand{\cons}[2]{\ensuremath{r_{#1}^{#2}}}
\newcommand{\limit}[2]{\ensuremath{s_{#1}^{#2}}}
\newcommand{\allocvec}[1]{\ensuremath{{\vec{v}}^{#1}}}
\newcommand{\R}{\ensuremath{\mathbb{R}}}
\newcommand{\play}[3]{\ensuremath{x_{#1,#2}^{#3}}}
\newcommand{\comp}{\ensuremath{u}}
\newcommand{\KL}{\ensuremath{\mbox{KL}}}
\newcommand{\half}{\ensuremath{\frac{1}{2}}}
\newcommand{\eps}{\ensuremath{\epsilon}}
\newcommand{\E}{\ensuremath{\mathbb{E}}}
\renewcommand{\P}{\ensuremath{\mathbb{P}}}

\newcommand{\ind}{\ensuremath{\mathbb{I}}}
\newcommand{\fl}[2]{\ensuremath{f_{#1}^{#2}}}
\newcommand{\diff}[2]{\ensuremath{d_{#1}^{#2}}}
\newcommand{\diffb}[2]{\ensuremath{{\bar{d}}_{#1,#2}}}
\newcommand{\intvol}[2]{\ensuremath{\tilde{V}_{#1,#2}}}
\newcommand{\gtil}{\ensuremath{{\tilde{g}}}}
\newcommand{\rew}{\ensuremath{\mathbf{r}}}
\newcommand{\unif}{\ensuremath{{\mbox{unif}}}}
\newcommand{\gh}{\ensuremath{{\hat{g}}}}
\newcommand{\var}{\ensuremath{\mbox{Var}}}

\newcommand{\expgrad}{\textsc{ExpGrad}}
\newcommand{\expthree}{\textsc{Exp3}}
\newcommand{\optkm}{\textsc{OptKM}}
\newcommand{\parml}{\textsc{ParML}}

\newtheorem{theorem}{Theorem}
\newtheorem{lemma}{Lemma}
\newtheorem{corollary}{Corollary}

\begin{document}

%\twocolumn[
\title{Optimal Allocation Strategies for the Dark Pool Problem}
\author{Alekh Agarwal \\University of California,
  Berkeley\\\texttt{alekh@cs.berkeley.edu}\and Peter Bartlett \\University of
  California, Berkeley\\\texttt{bartlett@cs.berkeley.edu} \and Max Dama\\University
  of California, Berkeley\\\texttt{maxdama@berkeley.edu}}
%]

%% \addtolength{\abovedisplayskip}{-0.5cm}
%% \addtolength{\belowdisplayskip}{-0.5cm}

\maketitle

\begin{abstract}
  We study the problem of allocating stocks to \emph{dark pools}. We
  propose and analyze an optimal approach for allocations, if
  continuous-valued allocations are allowed. We also propose a
  modification for the case when only integer-valued allocations are
  possible. We extend the previous work on this
  problem~\citep{gknv2009darkpools} to adversarial scenarios, while
  also improving on their results in the iid setup. The resulting
  algorithms are efficient, and perform well in simulations
  under stochastic and adversarial inputs.
\end{abstract}

\section{Introduction}

In this paper we consider the problem of allocating stocks to \emph{dark
pools}. As described by \citep{gknv2009darkpools}, dark pools are
a recent type of stock exchange that are designed to facilitate large
transactions. A key aspect of dark pools is the \emph{censored
  feedback} that the trader receives. At every round the trader has a
certain number $\vol{t}$ of shares to allocate amongst $K$ different
dark pools. The dark pool $i$ trades as many of the allocated shares
$v_i$ as it can with the available liquidity. The trader only finds
out how many of these allocated
shares were successfully traded at each dark pool, but not how many
would have been traded if more were allocated.

It is natural to assume that the actions of the trader affect the
volume available at all dark pools
at later times.  Similarly, it seems natural that at a given time, the
liquidities available at different venues should be correlated:
we would expect counterparties to distribute large trades across many
dark pools, simultaneously affecting their liquidity.
Furthermore, in a realistic scenario, these variables are
governed not only by the trader's actions, but also by the
actions of other
competing traders, each trying to maximize profits. Since the
gain of one trader is at the expense of another, this
problem naturally lends itself to an adversarial analysis. Generalizing
the setup of \citep{gknv2009darkpools}, we assume that
the sequences of volumes and
available liquidity at each venue are chosen
by an adversary who knows the
previous allocations of our algorithm. 

We propose an exponentiated gradient (henceforth EG) style
algorithm that has an optimal regret guarantee against the best
allocation strategy in hindsight. Our algorithm uses a parametrization that
allows it to handle the problem of changing constraint sets
easily. Through a standard online to batch conversion, 
this also yields a significantly better algorithm in the iid setup
studied in \citep{gknv2009darkpools}. However, the
EG algorithm has the drawback that it recommends
continuous-valued allocations. %%  Of course in the dark pool
%% problem we would like to allocate an integral number of shares.
We describe how the problem of allocating
an integral number of shares closely resembles a multi-armed
bandit problem. 
As a result, we use ideas from the Exp3 algorithm for adversarial
bandit problems~\citep{auer2003exp3} to design an
algorithm that produces
integer-valued allocations and enjoys a regret of order
$T^{2/3}$ with high probability. While this regret bound holds in an
adversarial setting, it also implies an improvement on
\citep{gknv2009darkpools} in an iid setting.
We also study an
efficient implementation of our algorithm using 
the idea of greedy approximations in Hilbert spaces
\citep{jones92greedy}, \citep{barron93approx}.

In the next section we will describe the problem setup in more detail
and survey previous work. We will describe the EG
algorithm for continuous allocations and prove its regret bound and
optimality in Section~\ref{sec:cont}. In Section~\ref{sec:int} we
describe the algorithm for integer valued
allocations. Section~\ref{sec:greedy} describes an efficient
implementation. Finally we
present experiments comparing our algorithms with that
of \citep{gknv2009darkpools} using the data simulator described in 
their paper.

\section{Setup and Related Work}
We  generalize the setup of \citep{gknv2009darkpools}. A
learning algorithm receives a sequence of volumes
$\vol{1},\dots\,\vol{T}$ where $\vol{t} \in \{1,\dots,\vol{}\}$. It has $K$
available venues, amongst which it 
can allocate up to $\vol{t}$ units at time $t$. The learner chooses an
allocation $\alloc{i}{t}$ for the $i_{th}$ venue at time $t$ that
satisfies $\sum_{i=1}^K v_i^t \le V^t.$
%%   \[
%%     \sum_{i=1}^K v_i^t \le V^t.
%%   \]\

Each
venue has a maximum consumption level $\limit{i}{t}$. The
learner then receives the number of units $\cons{i}{t} =
\min(\alloc{i}{t}, \limit{i}{t})$ consumed at venue $i$. We allow the
sequence of volumes 
and maximum consumption levels to be chosen adversarially, i.e.
$V_t, \limit{i}{t}$ can depend on
$\{\alloc{i}{1},\dots,\alloc{i}{t-1}\}_{i=1}^K$. We measure the
performance of our learner in terms of its regret

\begin{small}
  $$
    R_T = \max \sum_{t=1}^T\sum_{i=1}^K\min(u_i^t,\limit{i}{t}) -
      \min(\alloc{i}{t},\limit{i}{t})
  $$
\end{small}

where the outer maximization is over the vector
$\mbox{opt}\in\{1,\ldots,K\}^V$ and
  \[
    u_i^t = \sum_{v=1}^{V^t} \ind(\mbox{opt}_v = i),
  \]

i.e., we compete against any strategy that chooses
a fixed sequence of venues $\mbox{opt}_1,\ldots, \mbox{opt}_V$
and always allocates the $v$th unit to
venue $\mbox{opt}_v$.

The work most closely related to ours is \citep{gknv2009darkpools}. In
that paper, the authors consider the 
sequence of volumes $\vol{1},\dots,\vol{T}$ and allocation limits
$\limit{i}{t}$ to be distributed in an 
iid fashion. They propose an
algorithm based on Kaplan-Meier estimators.
Their algorithm 
mimics an optimal allocation strategy by estimating the tail
probabilities of $\limit{i}{t}$ being larger than a given value. They
show that the allocations of their algorithm are $\epsilon$-suboptimal
with probability at most $1-\epsilon$ after seeing sufficiently many
samples. 
Theorem~1 in~\citep{gknv2009darkpools} shows that, if the $s_i^t$ is
chosen iid, then the optimal strategy always allocates the $i$th unit
to a fixed venue. This justifies our definition of
regret in comparison to this class of strategies.

The ideas used in our paper draw on the rich literature on
online adversarial learning. The algorithm of Section~\ref{sec:cont}
is based on the classical EG algorithm
~\citep{lw94wm}. When playing integral allocations, we describe how the
multi-armed bandits problem is a special case of our problem for
$V=1$. For the general case, we describe an adaptation of the Exp3
algorithm~\citep{auer2003exp3} for adversarial multi-armed bandits.
To provide regret bounds that hold with high probability, we use a
variance correction similar to
the Exp3.P algorithm~\citep{auer2003exp3}. Our lower bounds use
information theoretic techniques, building on Fano's
method~\citep{yu93assouad}. The efficient implementation of our
algorithm relies on greedy approximation techniques in Hilbert space
\citep{jones92greedy}, \citep{barron93approx}.

\addtolength{\abovedisplayskip}{0.3cm}
\addtolength{\belowdisplayskip}{0.3cm}

\section{Optimal algorithm for fractional allocations}
\label{sec:cont}
Although the dark pool problem requires us to allocate an integral
number of shares at every venue, we start by studying the
simpler case where we can allocate any positive value for every venue,
so long as they satisfy $\sum_{i=1}^K\alloc{i}{t} \leq \vol{t}$. We
start by noting that the reward function
$\cons{i}{t}=\min(\alloc{i}{t},\limit{i}{t})$ is concave in 
allocations $\alloc{i}{t}$. 
%% \begin{fact}
%%   $\cons{i}{t} = \min(\alloc{i}{t},\limit{i}{t})$ is a concave function
%%   of $\alloc{i}{t}$.
%% \end{fact}

Maximization of concave functions is well understood, even in an
adversarial scenario through approaches such as online gradient
ascent. We note that in this problem, the algorithm has access to the
subgradient of the reward function. To see this, we define
%% subgradient set of $\cons{i}{t}$ wrt $\alloc{i}{t}$.
%% \begin{equation*}
%%   \frac{\partial \cons{i}{t}}{\partial \alloc{i}{t}} =
%%   \left\{\begin{array}{cc} \{1\}&\mbox{if}~\cons{i}{t} = \alloc{i}{t}
%%   < \limit{i}{t}\\\{0\}&\mbox{if}~\cons{i}{t} = \limit{i}{t} <
%%   \alloc{i}{t}\\\left[0,1\right]&\mbox{if}~\cons{i}{t} = \alloc{i}{t} =
%%   \limit{i}{t}\end{array}\right. 
%% \end{equation*}

%% Hence, we can set 
\begin{equation}
  g_i^t = \left\{\begin{array}{cc}1&\mbox{if }\cons{i}{t} =
  \alloc{i}{t}\\0&\mbox{if }\cons{i}{t} <
  \alloc{i}{t}\end{array}\right. 
  \label{eqn:gradient}
\end{equation}

Then it is easy to check that $g_i^t$ can be constructed from the
feedback we receive, and it lies in the subgradient set
$\frac{\partial \cons{i}{t}}{\partial \alloc{i}{t}}$. Hence, we can run a
standard online (sub)gradient ascent algorithm on 
this sequence of reward functions. However, the allocations
$\alloc{i}{t}$ are chosen from a different set $S_t = 
\{\allocvec{t}~:~ \sum_{i=1}^K\alloc{i}{t} \leq \vol{t}\}$ at every
round. Using standard online gradient ascent analysis, we can
demonstrate a low regret only against a comparator that lies in the
intersection of all these constraint sets
$\cap_{t=1}^TS_t$. However the regret guarantee can be rather
meaningless if $\vol{t}$ is extremely small at even a single
round. Ideally, we would like to compete with an optimal allocation
strategy like \citep{gknv2009darkpools}. A slightly
different parameterization allows us to do exactly that.

Let us define $\Delta_K^V = \{x^1,\dots,x^V~:~\sum_{i=1}^Kx_i^v =
1~\forall v \leq V\}$ to be the Cartesian product of $V$ simplices,
each in $\R^K$. Then we can construct an algorithm for allocations as
follows: for each unit $v = \{1,\dots,V\}$, we have a distribution
over the venues $\{1,\dots,K\}$ where that unit is allocated. At time $t$, the
algorithm plays $\alloc{i}{t} = \sum_{v=1}^{\vol{t}}
\play{t}{i}{v}$. It is clear that this allocation satisfies the volume
constraint. 

The comparator is now defined as a fixed point $\comp \in
\Delta_K^V$. We compete with the strategy that plays according to
$\alloc{i}{t} = \sum_{v=1}^{\vol{t}} \comp_i^v$. Then the best
comparator $\comp$ is equivalent to the best fixed allocation strategy
$\mbox{opt}\in\{1,\ldots,K\}^V$.
It is also clear that if we can compete
with the best strategy in an adversarial setup, online to
batch conversion techniques (see Cesa-Bianchi et
al~\citep{cbcg2001generalization}) will give a small expected error in
the case where the volumes and maximum consumptions are drawn in an
iid fashion. 

\subsection{Algorithm and upper bound}

An online gradient ascent algorithm for this setup is presented in
Algorithm~\ref{alg:fracalloc}. 

\begin{algorithm}[htb]
\begin{algorithmic}
  \STATE \textbf{Input} learning rate $\eta$, bound on volumes $V$.
  \STATE Initialize $\play{1}{i}{v} = \frac{1}{K}$ for $v \in
  \{1,\dots,V\}$, $i\in\{1,\ldots,K\}$.
  \FOR{$t=1,\dots,T$}
  \STATE Set $\alloc{i}{t} = \sum_{v=1}^{\vol{t}} \play{t}{i}{v}$. 
  \STATE Receive $\cons{i}{t} = \min\{\alloc{i}{t},\limit{i}{t}\}$. 
  \STATE Set $g_i^t$ as defined in Equation~(\ref{eqn:gradient}).
  \STATE Set $g_{t,i}^v = g_i^t$ if $v \leq \vol{t}$, 0 otherwise.
  \STATE Update $\play{t+1}{i}{v} \propto \play{t}{i}{v}\exp(\eta
  g_{t,i}^v)$. 
  \ENDFOR
\end{algorithmic}
\caption{Exponentiated gradient algorithm for continuous-valued allocations
  to dark pools}
\label{alg:fracalloc}
\end{algorithm}

It can be shown that the algorithm enjoys the following regret
guaranteee.
\begin{theorem}
  For any choices of the volumes $\vol{t}\in[0,V]$
  and of the maximum consumption levels $\limit{i}{t}$,
  the regret of Algorithm~\ref{alg:fracalloc} with $\eta =
  \sqrt{\frac{\ln K}{(e-2)T}}$ over $T$ rounds is
  $O(V\sqrt{T\ln K})$.
  \label{thm:regretfrac}
\end{theorem}

\begin{proof}
  The regret is defined as
  \begin{small}
  \begin{align*}
    R_T &= \max_{u \in \Delta_K^V}
    \sum_{t=1}^T\sum_{i=1}^K\min\left(\sum_{v=1}^{\vol{t}}\comp_i^v,
    \limit{i}{t}\right) -
    \sum_{t=1}^T\sum_{i=1}^K\min\left(\alloc{i}{t},\limit{i}{t}\right)\\ 
    &\leq \sum_{t=1}^T\sum_{v=1}^{\vol{t}}\left(\comp^v
    - x_t^v\right)^\top g_t^v.
  \end{align*}
\end{small}
Following the proof of Theorem 11.3 from Cesa-Bianchi et
  al~\citep{cbl2006plg}, we define $\nu_i^v = \eta g_{t,i}^v - \eta
  (g_t^v)^\top x_t^v.$ 
  Also, we note that the gradient is zero for $v > \vol{t}$. So we can
  sum over $v$ from $1$ to $V$ rather than $\vol{t}$. Then we bound
  the regret as
  \begin{align*}
    &\sum_{t=1}^T\sum_{v=1}^V\left[(u^v - x_t^v)^\top g_t^v
      -\frac{1}{\eta}\ln\left(\sum_{i=1}^K\play{t}{i}{v}\exp(\nu_i^v)\right)\right.\\ &\left. +
      \frac{1}{\eta}\ln\left(\sum_{i=1}^K\play{t}{i}{v}\exp(\nu_i^v)\right)
      \right]. %% \\
%%     &= \sum_{t=1}^T\sum_{v=1}^V\left[(g_t^v)^\top u^v -
%%     \frac{1}{\eta}\ln\left(\sum_{i=1}^K\play{t}{i}{v}\exp\left(\eta
%%     g_{t,i}^v\right)\right)\right. \\
%%     &+\left.
%%     \frac{1}{\eta}\ln\left(\sum_{i=1}^K\play{t}{i}{v}\exp\left(\nu_i^v\right)\right)\right]
  \end{align*}
%%   Here the last line follows from the definition of $\nu$ and using
%%   $\sum_{i=1}^K\play{t}{i}{v} = 1$. Further writing each coordinate of
%%   the first term as $\frac{1}{\eta}u_i^v\ln\exp\left(\eta
%%   g_{t,i}^v\right)
Some rewriting and simplification gives the bound
  \begin{small}
  \begin{align*}
    &\frac{1}{\eta}\sum_{t=1}^T\sum_{v=1}^V\left[
      \sum_{i=1}^Ku_i^v\ln\left( \frac{\exp\left(\eta
        g_{t,i}^v\right)}{\sum_{i=1}^K\exp\left(\eta g_{t,i}^v\right)} \right)
      +%\right.\\&\left.+ 
      \ln\left(\sum_{i=1}^K \play{t}{i}{v}e^{\nu_i^v}\right)\right] \\
    &= \frac{1}{\eta}\sum_{t=1}^T\sum_{v=1}^V\left[ u_i^v\ln\left(
      \frac{\play{t+1}{i}{v}}{\play{t}{i}{v}}\right) +
      \ln\left(\sum_{i=1}^K \play{t}{i}{v}\exp(\nu_i^v)\right)\right]\\
%%     &= \frac{1}{\eta}\sum_{t=1}^T\sum_{v=1}^V\left[ \KL(u^v||x_t^v) -
%%       \KL(u^v||x_{t+1}^v)\right.\\ &+ \left.\ln\left(\sum_{i=1}^K
%%       \play{t}{i}{v}\exp(\nu_i^v)\right)\right]\\
    &\leq \frac{1}{\eta}\sum_{v=1}^V\left[\KL(u^v||x_1^v) + \sum_{t=1}^T\ln\left(\sum_{i=1}^K
      \play{t}{i}{v}\exp(\nu_i^v)\right)\right].
  \end{align*}
\end{small}
Here, the last line uses the definition of KL-divergence and
  the fact that the telescoping terms cancel out. Now $g_{t,i}^v \leq 1$
  so that $\nu_i^v \leq \eta$. If $\eta \leq 1$, then it is easy to
  verify that $\exp(\nu_i^v) \leq 1 + \nu_i^v +
  (e-2)\left(\nu_i^v\right)^2.$
  We also note that  $\sum_{i=1}^K \play{t}{i}{v}\nu_i^v = 0.$ 

  Also, each of the KL divergence terms in the above display is equal
  to $\ln K$. This is because the optimal comparator will have a 1
  for exactly one venue for each unit $v$. As we choose $x_1^v$ to be
  uniform over all venues, we get the KL divergence between a vertex
  of the $K$-simplex and the uniform distribution which, is $\ln K$.

  Hence we bound the regret as
  \begin{align*}
    &\frac{1}{\eta} V\ln K +
    \frac{1}{\eta}\sum_{t=1}^T\sum_{v=1}^V\ln\left(\sum_{i=1}^K
    \play{t}{i}{v}\left(1 + \nu_i^v + (e-2)\left(\nu_i^v\right)^2\right) \right)\\
    %% &\leq \frac{1}{\eta} V\ln K +
%%     \frac{1}{\eta}\sum_{t=1}^T\sum_{v=1}^V\ln\left(1 + \sum_{i=1}^K
%%     \play{t}{i}{v}\left(\nu_i^v + (e-2)\left(\nu_i^v\right)^2\right) \right)\\
    %% &\leq \frac{1}{\eta} V\ln K +
%%     \frac{1}{\eta}\sum_{t=1}^T\sum_{v=1}^V\left(\sum_{i=1}^K
%%     \play{t}{i}{v}\nu_i^v + (e-2)\left(\nu_i^v\right)^2 \right)\\
    &\leq \frac{1}{\eta} V\ln K +
    \frac{1}{\eta}\sum_{t=1}^T\sum_{v=1}^V(e-2)\eta^2\\
    &= \frac{1}{\eta}V\ln K + (e-2)\eta VT\\
    &\leq 3V\sqrt{T\ln K},
  \end{align*}
  where the last step follows from setting $\eta = \sqrt{\frac{\ln
      K}{(e-2)T}}$. 
\end{proof}

\subsection{Lower bound and minimax optimality}
\label{sec:lbcont}

We will now show that the online exponentiated gradient ascent algorithm in
Algorithm~\ref{alg:fracalloc} has the best regret guarantee
possible. We start by noting that a a regret bound of $O(\sqrt{T\ln
  K})$ is known to be optimal for the experts prediction
problem~\citep{hkw98optimal, aabr-svormd-09}. Hence we can show the
optimality of our algorithm for $V=1$ by reducing experts prediction
problem to the dark pools problem. Recall that in the experts
prediction problem, the algorithm picks an expert from $1,\dots,K$
according to a probability distribution $p_t$ at round $t$. Then it
receives a vector of rewards $\rho_t$ with $\rho_{t,i} \in
[0,1],~~i=1,\dots,K$. In order to describe a reduction, we need to
map the allocations of an algorithm for the dark pools problem to the
probabilities for experts, and map the rewards of experts to the
liquidities at each venue. 

We consider a special setting where $V_t = 1$ at all times. Since
$V_t=1$, the allocations of any dark pools algorithm are 
probabilities-- they are non-negative and add to 1. Hence we set
$p_{t,i} = \alloc{i}{t}$. We also set the liquidity $\limit{i}{t} =
\rho_{t,i}p_{t,i}$. Then the net reward of a dark pools algorithm at
round $t$ is: 
$$\sum_{i=1}^K\min(\limit{i}{t},\alloc{i}{t}) =
\sum_{i=1}^K\min(\rho_{t,i}p_{t,i},p_{t,i}) =
\sum_{i=1}^K\rho_{t,i}p_{t,i},$$ where the last line follows from the
observation that $0 \leq \rho_{t,i} \leq 1$. Hence the net reward of the
dark pools problem is same as that expected reward in the experts
prediction problem. Using the known lower bounds on the optimal regret
in experts prediction problems, we get:
\begin{align*}
&\max_{u \in \Delta_K} \sum_{t=1}^T
\sum_{i=1}^K\left[\min\left(\comp_{i},\limit{i}{t}\right) -
\min(\alloc{i}{t},\limit{i}{t})\right] 
\\
&= \max_i\sum_{t=1}^T\left[\rho_{t,i} - \sum_{j=1}^K\rho_{t,j}p_{t,j}\right]\\
&= \Omega(\sqrt{T\ln K}).
\end{align*}

We also note that the regret in the experts prediction problem scales
linearly with the scaling of the rewards. Hence, if the rewards take
values in $[0,V]$, then the regret of any algorithm is guaranteed to
be $\Omega(V\sqrt{T\ln K})$.

For arbitrary $V$, we again consider the special setting with $V_t$
identically equal to $V$. We would now like to reduce the experts
prediction problem where every expert's reward is a value in
$[0,V]$. At every round, we receive a vector of 
allocations $\alloc{i}{t}$. We set $p_{t,i} = \alloc{i}{t}/V$. We
receive the rewards $\rho_{t,i}$ from the experts problem, and assign
the liquidities $\limit{i}{t} = \rho_{t,i}p_{t,i} \in
[0,V]$. Furthermore, 
$$\min(\limit{i}{t},\alloc{i}{t}) =
V\min\left(\frac{\limit{i}{t}}{V},p_{t,i}\right) =
\rho_{t,i}p_{t,i}.$$ The last step relies on observing that
$\rho_{t,i} \leq V$ so that $\rho_{t,i}p_{t,i}/V \leq p_{t,i}$. Now we
can argue that the regrets of the two problems are identical as
before. Hence the optimal regret on the dark pools problem is at least
$\Omega(V\sqrt{T\ln K})$. As Algorithm~\ref{alg:fracalloc} gets the
same bound up to constant 
factors in a harder adversarial setting than used in the lower bounds,
we conclude that it attains the minimax optimal
regret up to constant factors.

\section{Algorithm for integral allocations}
\label{sec:int}
While the above algorithm is simple and optimal in theory, it is a bit
unrealistic as it can recommend we allocate 1.5 units to a venue, for
example. One
might choose to naively round the recommendations of the algorithm,
but such a rounding would incur an additional approximation error
which in general could be as large as $O(T)$. In this section we
describe a low regret algorithm that allocates an
integral number of units to each venue.

To get some intuition about an algorithm for this scenario, consider
the case when $V=1$. Then the algorithm has to allocate 1 unit to a
venue at every round. It receives feedback about the maximum
allocation level $\limit{i}{t}$ only at the venue where $\alloc{i}{t}
= 1$. This is clearly a reformulation of the classical $K$-armed
bandits problem. An adaptation of Algorithm~\ref{alg:fracalloc} that
uses the Exp3 algorithm~\citep{auer2003exp3} would hence attain a
regret bound of $O(\sqrt{TK\ln K})$ for $V=1$. Contrasting this with
the bound of Theorem~\ref{thm:regretfrac} for $V = 1$, we can easily
see that the regret for playing integral allocations can be higher
than that of continuous allocations by a factor of up to
$\sqrt{K}$. Indeed we will now show a modification of the Exp3
approach that works for arbitrary values of $V$. We will also show a
lower bound. The upper bound shows that our algorithm incurs
$O(T^{2/3})$ regret in
expectation, which does not match the $O(\sqrt{T})$ lower bound. However, it
is still a significant improvement on Ganchev et
al~\citep{gknv2009darkpools} as we will discusss later.

\subsection{Algorithm and upper bound}

We need some new notation before describing the algorithm. For a
fractional allocation $\alloc{i}{t}$, we let $\fl{i}{t} =
\lfloor\alloc{i}{t}\rfloor$  and $\diff{i}{t} = \alloc{i}{t} -
\lfloor\alloc{i}{t}\rfloor$. 

Now suppose we have a strategy that wants to allocate $\alloc{i}{t}$
units to venue $i$ at time $t$. Suppose that we instead allocate
$\alloci{i}{t} = \fl{i}{t}$ units with probability $1-\diff{i}{t}$ and
$\alloci{i}{t} = \fl{i}{t} + 1$ 
units with probability $\diff{i}{t}$. Using the fact that the maximum
consumption limits are integral too
\begin{align*}
  \E\min(\alloci{i}{t},\limit{i}{t})
  &= \diff{i}{t}\min(\fl{i}{t}+1,\limit{i}{t}) + (1 - \diff{i}{t})\min(\fl{i}{t},\limit{i}{t})\\
  &= \left\{\begin{array}{cc} \limit{i}{t}&\mbox{if}~\limit{i}{t}
  \leq \fl{i}{t}\\\fl{i}{t} + \diff{i}{i}&\mbox{if}~\limit{i}{t} \geq
  \fl{i}{t}+1\end{array}\right.\\
  &= \min(\alloc{i}{t},\limit{i}{t}).
\end{align*}
Thus, playing an integral allocation $\alloci{i}{t}$ according to such
a scheme would be unbiased in expectation. Of course we need to ensure
that we don't violate the constraint $\sum_{i=1}^K\alloci{i}{t} \leq
\vol{t}$ in this process. To do so, we let $\sum_{i=1}^K\diff{i}{t} =
V_t - \sum_{i=1}^K\fl{i}{t} = m$. Then we will use a
distribution over subsets of $\{1,\dots,K\}$ of size $m$ that has the
property that $i_{th}$ element gets sampled with probability
$\diff{i}{t}$. It is clear that if there is such a distribution, then
we will have the unbiasedness needed above. It will also ensure
feasibility of $\alloci{i}{t}$ if $\alloc{i}{t}$ was a feasible
allocation. Our next result shows that such a distribution always
exists. 
\begin{theorem}
  Let $0 \le \diff{i}{t} < 1,~\sum_{i=1}^K\diff{i}{t} = m$ for $m \ge
  1$. Then there 
  is always a distribution over subsets of 
  $\{1,\dots,K\}$ of size $m$ such that the $i_{th}$ element is
  sampled with probability $\diff{i}{t}$. 
  \label{thm:probexist}
\end{theorem}

\begin{proof}  
  Proof is by induction on $K$. For the case $K=2,m=1$, we sample the
  first element with probability $\diff{1}{t}$. If it is not picked,
  we pick element 2. It is clear that the marginals are correct
  establishing the base case. Let us assume the claim holds up to
  $K-1$ for all $m 
  \leq K-1$. Consider the inductive step for some $K,m$. We are
  given a set of marginals, $0 \le \diff{i}{t} <
  1,~\sum_{i=1}^K\diff{i}{t} = m$. We would like a distribution $p$ on
  subsets of size $m$ of $\{1,\dots,K\}$ that matches these
  marginals. We partition these subsets into two groups; those that do
  and do not contain the first element. We correspondingly partition
  $p = (p_1,p_2)$. Let $N_1 = \binom{K-1}{m-1}$ and $N_2 =
  \binom{K-1}{m}$ be the number of subsets in the two cases. Then we
  want $\sum_{i=1}^Np(i) = 
  \sum_{i=1}^{N_1}p_1(i) = \diff{1}{t}$ in order to get the right marginal at
  element 1. Hence, we can write $p_1 = \diff{1}{t}q_1$, $p_2 =
  (1-\diff{1}{t})q_2$ for some distributions $q_1$ and $q_2$ on $N_1$
  and $N_2$ subsets respectively. Now we write
  \begin{equation}
    \diff{i}{t} = \left(\frac{(m-1)\diff{1}{t}}{m-\diff{1}{t}} +
    \frac{m(1-\diff{1}{t})}{m-\diff{1}{t}}\right)\diff{i}{t} \label{eqn:mixture}\end{equation}
  for $i > 1$. Then 
  \begin{align}
    \sum_{i=2}^K\frac{(m-1)}{m-\diff{1}{t}}\diff{i}{t} =
    m-1,~~\sum_{i=2}^K\frac{m}{m-\diff{1}{t}}\diff{i}{t} =
    m
    \label{eqn:rightmarg}
  \end{align}
  are marginals on subsets of size $m-1$ and $m$ respectively of
  $\{1,\dots,K-1\}$, and are in $[0,1]$ as $\sum_{i=2}^K\diff{i}{t} =
  m-\diff{1}{t}$. Hence there exist distributions $q_1$ and $q_2$ 
  that attain these marginals using the inductive hypothesis. We set
  $p_1 = \diff{1}{t}q_1$, $p_2 = (1-\diff{1}{t})q_2$. Then
  Equations~\ref{eqn:mixture} and \ref{eqn:rightmarg} together imply
  that we get the correct marginals for every element. 
\end{proof}

For any allocation sequence $\alloc{}{t}$, let $p(\diff{}{t})$ be the
probability distribution over subsets of $\{1,\dots,K\}$ guaranteed by
Theorem~\ref{thm:probexist}. For some constant $\gamma \in (0,1]$, let
  $\diffb{t}{i} = (1-\gamma)\diff{i}{t} + \frac{\gamma m}{K}$. Then let
  $p(\diffb{t}{i})$ be a distribution over subsets that samples the
  $i_{th}$ venue with probability $\diffb{t}{i}$. We can construct
  this by mixing $p(\diff{i}{t})$ which exists by
  Theorem~\ref{thm:probexist} and mixing uniform distribution over
  subsets of size $m$. Also, we let
  $\intvol{t}{i} \leq V_t$ be 
 the largest index $v_0$ such that $\sum_{v=1}^{v_0} \play{t}{i}{v} \leq
\fl{i}{t}$. We define a gradient estimator:
\begin{equation}
  \gtil_{t,i}^v = \left\{\begin{array}{ccc} \ind(\limit{i}{t} \geq
  \fl{i}{t}) - \frac{\ind(\limit{i}{t} = \fl{i}{t})\ind(\alloci{i}{t} =
  \lceil\alloc{i}{t}\rceil)}{\diffb{t}{i}}~\mbox{if}~v \leq
  \intvol{t}{i}\\ \frac{\ind(\limit{i}{t} \geq
    \alloc{i}{t})\ind(\alloci{i}{t} = 
    \lceil\alloc{i}{t}\rceil)}{\diffb{t}{i}}~\mbox{if}~\intvol{t}{i} + 1
  \leq v \leq \vol{t}.\end{array}\right.
  \label{eqn:gradest}
\end{equation}

To see why this gradient estimator is good, we first note that the
gradient of the objective function at $\alloc{i}{t}$ can be written as 
$$g_{t,i}^v = \ind(\limit{i}{t} \geq \alloc{i}{t}) =
\ind(\limit{i}{t} \geq \fl{i}{t}) - \ind(\limit{i}{t} = \fl{i}{t}),$$
when $v \leq \vol{t}$. Then we can easily show the following useful
lemma. 
\begin{lemma}
  If an algorithm plays $\alloci{i}{t} = \lceil\alloc{i}{t}\rceil$
  with probability $\diffb{t}{i}$ and $\alloci{i}{t} = \fl{i}{t}$
  otherwise, then $\gtil_t$ as described in
  Equation~(\ref{eqn:gradest})
  is an unbiased estimator of the gradient at
  $(\alloc{1}{t},\dots,\alloc{K}{t})$. 
\end{lemma}

An algorithm for playing
integer-valued allocations at every round is shown in
Algorithm~\ref{alg:intallocexp}.
\begin{algorithm}[htb]
  \begin{algorithmic}
    \STATE \textbf{Input} learning rate $\eta$, threshold $\gamma$,
    bound on volumes $V$.
    \STATE Initialize $\play{1}{i}{v} = \frac{1}{K}$ for
    $v=\{1,\dots,V\}$. 
    \FOR{$t=1\dots T$}
    \STATE Set $\alloc{i}{t} = \sum_{v=1}^{\vol{t}} \play{t}{i}{v}$.
    \STATE Let $p(\diffb{t}{i})$ be the distribution over subsets from
    Theorem~\ref{thm:probexist}. 
    \STATE Sample a subset of size $m = \sum_{i=1}^K\diffb{t}{i}$
    according to $p(\diffb{t}{i})$.
    \STATE Play $\alloci{i}{t} = \fl{i}{t}+1$ if $i$ is in the subset
    sampled, $\alloci{i}{t} = \fl{i}{t}$ otherwise.
    \STATE Receive $\cons{i}{t} = \min(\alloci{i}{t},\limit{i}{t})$.
    \STATE Set $\gtil_{t,i}^v$ as defined in
    Equation~(\ref{eqn:gradest}). 
    \STATE Update $\play{t+1}{i}{v} \propto
    \play{t}{i}{v}\exp(\eta\gtil_{t,i}^v)$. 
    \ENDFOR
  \end{algorithmic}
  \caption{An algorithm for playing integer-valued allocations to the
    dark pools}
  \label{alg:intallocexp}
  \end{algorithm}

We can also demonstrate a guarantee on the expected regret of this
algorithm. 
\begin{theorem}
  Algorithm~\ref{alg:intallocexp}, with $\eta =
  \left(\frac{V(\ln K)^2}{KT^2}\right)^{1/3}$, has expected regret
  over $T$ rounds of $O((VTK)^{2/3}(\ln K)^{1/3})$, where $V$ is the
  bound on volumes $\vol{t}$, and the volumes and maximum consumption
  levels $\limit{i}{t}$ are chosen by an oblivious adversary.
  \label{thm:regretintexp}
\end{theorem}
An oblivious adversary is one that chooses $V^t$ and $s_i^t$
without seeing the algorithm's (random) allocations $u_i^t$.
We note that the requirement that the adversary is oblivious can be
removed by proving a high probability bound. We will describe a slight
modification of Algorithm~\ref{alg:intallocexp} that enjoys such a
guarantee.
\begin{proof}
  Since the adversary is oblivious, we can fix a comparator $\comp \in
  \Delta_K^V$ ahead of time. For the remainder, we let $\E_t$ denote
  conditional expectation at time $t$ conditioned on the past moves of
  algorithm and adversary. Then the expected regret is
  \begin{small}
  \begin{align*}    
    &\E\left[\sum_{t=1}^T\sum_{i=1}^K\min\left(\sum_{v=1}^V\comp_i^v,\limit{i}{t}\right)
      -
      \sum_{t=1}^T\sum_{i=1}^K\min\left(\alloci{i}{t},\limit{i}{t}\right)\right]\\
    &\leq
    \E\left[\sum_{t=1}^T\sum_{i=1}^K\min\left(\sum_{v=1}^V\comp_i^v,\limit{i}{t}\right)
      -
      \sum_{t=1}^T\sum_{i=1}^K\min(\alloc{i}{t},\limit{i}{t})\right] +
    \gamma TK.
  \end{align*}
  \end{small}
Here, the second step follows from the fact that $\alloci{i}{t}$
  would be unbiased for $\alloc{i}{t}$ without for the
  $\frac{\gamma m}{K}$ adjustment. However, this adjustment costs us
  at most $\gamma \sum_{t=1}^Tm_t \leq \gamma TK$ in terms of expected regret over $T$
  rounds. For the first term, it is as if we had played the continuous
  valued allocation $\alloc{i}{t}$ itself. Again using the concavity
  of our reward function
  \begin{align*}
    R_T(\comp) &\leq \E\left[\sum_{v=1}^V(\comp^v - x_t^v)^\top
      g_t^v\right] + \gamma TK\\
    &= \E\left[\sum_{v=1}^V(\comp^v -
      x_t^v)^\top(\E_t\gtil_t^v)\right] + \gamma TK.
  \end{align*}
  Here the last step follows from noting that $\gtil_t$ is unbiased
  estimator of $g_t$ by construction just like in
  Exp3~\citep{auer2003exp3}. Now we note that the algorithm is doing
  exponentiated gradient descent on the sequence $\gtil_t$. Hence, we
  can proceed as in the proof of Theorem~\ref{thm:regretfrac} to
  obtain 
  \addtolength{\abovedisplayskip}{0.1cm}
  \addtolength{\belowdisplayskip}{0.1cm}
  \begin{small}
  \begin{align*}
    R_T(\comp) & \leq \frac{1}{\eta} V\ln K +
    %& \qquad {} +
    \frac{1}{\eta}\E\sum_{t=1}^T\sum_{v=1}^V\ln\left(\sum_{i=1}^K
    \play{t}{i}{v}\exp(\nu_i^v)\right) + \gamma TK,
  \end{align*}\end{small}%
where $\nu_i^v = \eta\gtil_{t,i}^v - \eta(\gtil_t^v)^\top x_t^v$ as
  before. Assuming a choice of $\eta$ such that $\eta\gtil_{t,i}^v
  \leq 1$, we note again that $\nu_i^v \leq 1$. So we can use the
  quadratic bound on exponential again and simplify as before to get
  \addtolength{\abovedisplayskip}{-0.1cm}
  \addtolength{\belowdisplayskip}{-0.1cm}
  \begin{align*}
    R_T(\comp) &\leq \frac{1}{\eta}V\ln K +
    \frac{1}{\eta}\E\sum_{t=1}^T\sum_{v=1}^V\sum_{i=1}^K\play{t}{i}{v}(\nu_i^v)^2
    + \gamma TK\\
%%     &\leq \frac{1}{\eta}V\ln K +
%%     \frac{1}{\eta}\E\sum_{t=1}^T\sum_{v=1}^V \sum_{i=1}^K
%%     \play{t}{i}{v}(\eta\gtil_{t,i}^v)^2 + \gamma TK\\
    &= \frac{1}{\eta}V\ln K +
    \eta\E\sum_{t=1}^T\sum_{v=1}^V \sum_{i=1}^K
    \play{t}{i}{v}(\gtil_{t,i}^v)^2 + \gamma TK.
  \end{align*}
  Now we can swap the sum over $V$ and $i$ to obtain
  \begin{align*}
    R_T(\comp) &\leq \frac{1}{\eta} V\ln K +
    \eta\E\sum_{t=1}^T\sum_{i=1}^K\sum_{v=1}^V\play{t}{i}{v}(\gtil_{t,i}^v)^2
    + \gamma TK\\
    &= \frac{1}{\eta} V\ln K +
    \eta\E\sum_{t=1}^T\sum_{i=1}^K\left[\sum_{v=1}^{\intvol{t}{i}}\play{t}{i}{v}(\gtil_{t,i}^v)^2
      \right.\\&+\left. \sum_{v=\intvol{t}{i}+1}^{\vol{t}}\play{t}{i}{v}(\gtil_{t,i}^v)^2\right]
    + \gamma TK.
  \end{align*}
  Now we look at the two gradient terms separately.
  \begin{small}
  \begin{align*}
    \E_t\sum_{v=1}^{\intvol{t}{i}} \play{t}{i}{v}(\gtil_{t,i}^v)^2
    &= \sum_{v=1}^{\intvol{t}{i}}\play{t}{i}{v}\left\{
    \diffb{t}{i}\left(\ind(\limit{i}{t} \geq \fl{i}{t}) -
    \frac{\ind(\limit{i}{t} = \fl{i}{t})}{\diffb{t}{i}}\right)^2
    \right.\\&\quad+ (1-\diffb{t}{i})\ind(\limit{i}{t} \geq \alloc{i}{t})\bigg\}\\ 
    %% &\leq 2\sum_{v=1}^{\intvol{t}{i}}\play{t}{i}{v}\left\{1 +
%%     \frac{\ind(\limit{i}{t} = \fl{i}{t})}{\diffb{t}{i}}\right\}\\
    &\leq 2\alloc{t}{i} + 2\alloc{t}{i}\frac{K}{\gamma}.
  \end{align*}
  \end{small}
Here, we used the fact that $\diffb{t}{i} \geq \frac{\gamma}{K}$ 
  as $m \geq 1$ and indicator variables are bounded by 1. Hence
  \begin{align*}
    &\E\sum_{t=1}^T\sum_{i=1}^K\sum_{v=1}^{\intvol{t}{i}}
    \play{t}{i}{v}(\gtil_{t,i}^v)^2
    \leq 2TV + 2\frac{TVK}{\gamma}
  \end{align*}
  using $\sum_{i=1}^T\alloc{i}{t} \leq V$. Next we examine the second
  gradient term 
  \begin{align*}
    &\E_t\sum_{v=\intvol{t}{i}+1}^{\vol{t}} \play{t}{i}{v}(\gtil_{t,i}^v)^2
    =
    \E_t\sum_{v=\intvol{t}{i}+1}^{\vol{t}}\play{t}{i}{v}(\gtil_{t,i}^{\vol{t}})^2\\
    &= \E_t\diff{i}{t}(\gtil_{t,i}^{\vol{t}})^2
    \leq \diffb{t}{i}\diff{i}{t}\frac{1}{(\diffb{t}{i})^2}~~\leq 2
  \end{align*}
  if $\gamma \leq \half$. 

  Hence, %\begin{align*}
  $\E\sum_{t=1}^T\sum_{i=1}^K\sum_{v=\intvol{t}{i}+1}^{\vol{t}}
    \play{t}{i}{v}(\gtil_{t,i}^v)^2 \leq 2TK.$
  %\end{align*}
  Substituting the above terms in the bound, we get
  \begin{align*}
    R_t(u) &\leq \frac{1}{\eta}V\ln K + 2\eta\left(TV +
    \frac{TVK}{\gamma} + TK\right) + \gamma TK.
  \end{align*}
  Optimizing for $\eta,\gamma$ gives
  $$R_T(u) \leq 6(VTK)^{2/3}(\ln K)^{1/3}.$$
\end{proof}

We note that the term responsible for $O(T^{2/3})$ regret is
$\frac{\ind(\limit{i}{t} = \fl{i}{t})}{\diffb{t}{i}}$. While we assume
that this can accumulate at every round in the worst case, it seems
unlikely that the liquidity $\limit{i}{t}$ will be equal to
$\fl{i}{t}$ very frequently. In particular, if the $\limit{i}{t}$'s
are generated by a stochastic process, one can control this
probability using the distribution of $\limit{i}{t}$ and obtain
improved regret bounds. 

%% By using standard variance correction
%% techniques \citep{auer2003exp3}, \citep{ar2009highprob}, we can show a
%% similar bound with high probability. Combining it with a union bound
%% over all comparators allows us to extend the results of
%% Theorem~\ref{thm:regretintexp} to adaptive adversaries too. We omit
%% these standard steps for lack of space.

\subsection{Variance correction and High probability bound} 

We would like to show that the analysis of the previous section holds
not just in expectation but also with high probability. This has two
advantages. First, it tells us that on most random choices made by our
algorithm, it has a low regret. Further, the high
probability guarantee can be easily combined with a union bound to
give a regret bound for non-oblivious (adaptive) adversaries as well.

High probability bounds in bandit problems are often tricky because
even though the gradient estimator is unbiased, its variance is
typically large. Hence, using standard martingale concentration on the
estimator directly gives a worse $O(T^{3/4})$ regret bound. To
demonstrate a high probability guarantee of $O(T^{2/3})$, we need to
make a variance correction to our estimator $\gtil_t$. We define
\begin{equation}
  \gh_{t,i}^v = \gtil_{t,i}^v +
  \frac{10\gamma}{K\diffb{t}{i}}\sqrt{\ln\frac{1}{\delta}}.
  \label{eqn:gradestprob}
\end{equation}
The high probability analysis makes repeated use of the classical
Hoeffding-Azuma inequality as well as a version of Freedman's
inequality from Bartlett et al~\cite{bdhkrt2008highprob}. which we
state for completeness. 
inequality. 
\begin{lemma}[\textbf{Hoeffding-Azuma inequality}]
  Let $X_1,\dots,X_T$ be a martingale difference sequence. Suppose
  that $|Y_t| \leq c$ almost surely for all $t \in
  \{1,\dots,T\}$. Then for all $\delta > 0$,
  $$\P\left(\sum_{t=1}^Tx_t > \sqrt{2Tc^2\ln(1/\delta)}\right) \leq \delta.$$
  \label{lemma:haineq}
\end{lemma}

\begin{lemma}[\textbf{\citet{bdhkrt2008highprob}}]
  Let $X_1,\dots,X_T$ be a martingale difference sequence with $|X_t|
  \leq b$. Let 
  $$\var_tX_t = \var(X_t|X_1,\dots,X_{t-1}).$$
  Let $V = \sum_{t=1}^T\var_tX_t$ be the sum of conditional variances
  of $X_t$'s and $\sigma = \sqrt{V}$. Then we have, for any $\delta
  \leq 1/e$ and $T \geq 4$,
  $$\P\left(\sum_{t=1}^TX_t > 2\max\{2\sigma,b\sqrt{\ln(1/\delta)} \}
  \sqrt{\ln(1/\delta)}\right) \leq \delta\log_2T$$
  \label{lemma:freedman}
\end{lemma}

We will now prove a series of concentration results which will
immediately give the desired regret bound when put together. The steps
in our analysis closely resemble the technique of
\citet{ar2009highprob}. The first 
concentration lemma shows that the regret of the integral 
allocations is close to their continuous valued counterparts.

\begin{lemma}
  \begin{align*}\P\left(\exists i:\sum_{t=1}^T\min(\alloci{i}{t},\limit{i}{t}) -
  \sum_{t=1}^T\min(\alloc{i}{t},\limit{i}{t}) \right.\left.> V\sqrt{T\ln(K/\delta)}
  +  \gamma T/K\right)\leq \delta. \end{align*}
  \label{lemma:regconc}
\end{lemma}
\begin{proof}
  We apply Lemma~\ref{lemma:haineq} to the martingale difference
  sequence $X_t = \min(\alloci{i}{t},\limit{i}{t})
  -\E_t\min(\alloci{i}{t},\limit{i}{t})$. Then $|X_t| \leq V$. So 
  $$\P\left(\sum_{t=1}^T\min(\alloci{i}{t},\limit{i}{t}) -
  \sum_{t=1}^T\E_t\min(\alloci{i}{t},\limit{i}{t}) >
  V\sqrt{T\ln(1/\delta)}\right)\leq \delta.$$ 
  But we note that by construction
  \begin{align*}
    \E_t\min(\alloci{i}{t},\limit{i}{t}) &=
    \diffb{t}{i}\min(\fl{i}{t}+1,\limit{i}{t}) +
    (1-\diffb{t}{i})\min(\fl{i}{t},\limit{i}{t})\\
    &= \min(\fl{i}{t} + \diffb{t}{i},\limit{i}{t})\\
    &\leq \min(\fl{i}{t} + \diff{i}{t},\limit{i}{t}) +
    \frac{\gamma}{K}. 
  \end{align*}
  The statement of lemma then follows from the above inequality and a
  union bound over all $K$ venues.
\end{proof}

The next step is to show that the terms $\sum_{v=1}^V(\comp^v -
x_t^v)^\top\gh_t^v$ and $\sum_{v=1}^V(\comp^v - x_t^v)^\top g_t^v$ are
close. We proceed indirectly by first bounding the conditional
variances.

\begin{lemma}
  \begin{align*}
    \var_t\left\{(\gtil_t^v - g_t^v)^\top(\comp^v -
  x_t^v)\right\} 
   \leq 5\left[\sum_{i=1}^K\frac{u_i}{\diffb{t}{i}} +
  \sum_{i=1}^K\frac{(\play{t}{i}{v})^2}{\diffb{t}{i}}\right].
  \end{align*}
  \label{lemma:varbound}
\end{lemma}

We now combine this with Freedman's inequality to bound
$(\gtil_t^v - g_t^v)^\top(\comp^v - x_t^v)$. 
\begin{lemma}
 \begin{align*}
   \P\left(\sum_{t=1}^T\sum_{v=1}^V(\gh_t^v - g_t^v)^\top
     (\comp^v - x_t^v) 
   > 30\gamma TV\sqrt{\ln(1/\delta)} +
  2V\left(\frac{K^2}{\gamma^2}+1\right)\ln(1/\delta) 
   \leq 2V\delta\log_2T\right).
 \end{align*}
  \label{lemma:gradconc}
\end{lemma}
\begin{proof}
  We define the martingale $X_t = \sum_{v=1}^V(\gtil_t^v -
  g_t^v)^\top(\comp^v - x_t^v)$. Then $|X_t| \leq
  V\left(\frac{K}{\gamma}+1\right)$ by H\"older's inequality. Applying
  Hoeffding-Azuma inequality gives the result.  
\end{proof}
Finally, we also need to show that the size of the gradient estimator
which is controlled in expectation is also bounded with high
probability.
\begin{lemma}
  \begin{align*}
    \P\left(\sum_{t=1}^T\sum_{i=1}^K\alloc{i}{t}
    (\gh_{t,i}^{\vol{t}})^2 
  > 
  2\left(\frac{K^2}{\gamma^2}V+8\ln\frac{1}{\delta}\right)
  \sqrt{2T\ln(1/\delta)}\right) 
  \leq \delta.
  \end{align*}
  \label{lemma:gradbound}
\end{lemma}
\begin{proof}
  We define the martingale
  $X_t = \sum_{t=1}^T\sum_{i=1}^K\alloc{i}{t}((\gtil_{t,i}^{\vol{t}})^2 -
  \E_t\gtil_{t,i}^{\vol{t}})^2$. Then using the bound on $\gtil_t$,
  and the bound on expectation from proof of
  Theorem~\ref{thm:regretintexp}, 
  $X_t \leq 2\frac{K^2}{\gamma^2}V$. Application of Hoeffding-Azuma
  inequality gives the result. 
\end{proof}

We are now in a position to prove a high probability bound on the
regret of Algorithm~\ref{alg:intallocexp} when run with the gradient
estimator $\gh_t$ instead of $\gtil_t$.

\begin{theorem}
  With probability at least 1 - $\frac{1}{T}$, the regret of
  Algorithm~\ref{alg:intallocexp} using the gradient estimator $\gh_t$
  against oblivious adversaries is
  $\widetilde{O}\left(V(TK)^{2/3}\right)$.
  \label{thm:regretintprob}
\end{theorem} 

The proof essentially involves putting the lemmas together, along with
the full information analysis of the quantity $(\comp_u^v -
x_t^v)^\top\gh_t^v$.
\begin{proof}
  Using Lemma~\ref{lemma:regconc}, with probability at least
  1-$\delta/3$ 
  \begin{align*}
    &R_T =
    \sum_{t=1}^T\sum_{i=1}^K\min(\sum_{v=1}^{\vol{t}}\comp_i^v,\limit{i}{t})
    - \min(\alloci{i}{t},\limit{i}{t})\\
    &\leq \sum_{t=1}^T\sum_{i=1}^K
    \min(\sum_{v=1}^{\vol{t}}\comp_i^v,\limit{i}{t}) 
    - \min(\alloc{i}{t},\limit{i}{t}) +
    \sqrt{2T\ln\frac{3K}{\delta}} + \gamma T\\ 
    &\leq \sum_{t=1}^T\sum_{v=1}^{\vol{t}}(\comp^v - x_t^v)^\top g_t^v
    + \gamma T + \sqrt{2T\ln\frac{3K}{\delta}}.
  \end{align*}
  Invoking Lemma~\ref{lemma:gradconc}, with probability at least
  1-$2\delta/3$, 
  \begin{align*}
    R_T &\leq \sum_{t=1}^T\sum_{v=1}^{\vol{t}}(\comp^v - x_t^v)^\top
    \gtil_t^v + \gamma T + \sqrt{2T\ln\frac{3K}{\delta}} \\&+
    2V\left(\frac{K^2}{\gamma^2}+1\right)\sqrt{2T\ln(3/\delta)} +
    30\gamma TV\sqrt{\ln(1/\delta)}.
  \end{align*}

  Once again we note that we are doing exponentiated gradient descent
  on $\gh_t$ so that we get from proof of Theorem~\ref{thm:regretfrac}
  $$\sum_{t=1}^T\sum_{v=1}^{\vol{t}}(\comp^v - x_t^v)^\top \leq
  \frac{1}{\eta} V\ln K +
  \eta\E\sum_{t=1}^T\sum_{i=1}^K\sum_{v=1}^V\play{t}{i}{v}(\gtil_{t,i}^v)^2
  . $$ 

  Using Lemma 7 and setting $\delta =
  \frac{1}{T}$ gives the statement of the theorem on optimizing for
  $\gamma,\eta$. 
\end{proof}

Note that our regret analysis so far has been against a fixed
comparator. When the adversary adapts to player sequence, the
comparator is random as well and depends on player's moves. However,
the comparator consists of delta vectors for every unit
$v$. Hence, there are a total of $K^V$ possible comparators. Hence, we
can take a union bound over all the comparators as well, and this
increases our regret bound by a factor of $V\ln K$ at most. This gives
us the following corollary.

\begin{corollary}
  With probability at least 1 - $\frac{1}{T}$, the regret of
  Algorithm~\ref{alg:intallocexp} against adaptive adversaries is
  $\widetilde{O}\left(V^2(TK)^{2/3}\right)$.
\end{corollary} 

\noindent\textbf{Comparison with results of \citet{gknv2009darkpools}:}
We note that although our results are  
in the adversarial setup, the same results also apply to iid
problems. In particular, using online-to-batch conversion
techniques~\citep{cbcg2001generalization}, we can show that, after $T$
rounds, with high probability the allocations of our algorithm on each
round is within
$\widetilde{O}(V^2T^{-1/3}K^{2/3})$ of the optimal allocation.
This is a significant
improvement on the result of \citet{gknv2009darkpools}: it is straightforward to
check that the proof they provide gives a corresponding upper bound
no better than $O(T^{-1/4})$. As we shall see, the
generalization to adversarial setups leads to improved
performance in simulations.

\subsection{Lower bound on regret for integral allocations}

As mentioned in the previous section, the problem of $K$-armed bandits
is a special case of the dark pools problem with integral
allocations. Hence, we would like to leverage the proof techniques
from existing lower bounds on the optimal regret in the $K$-armed
bandits problem. As before we consider a special case with $V_t = V$
at every round. Following \citet{auer2003exp3}, we
construct $K$ different distributions for generating the liquidities
$\limit{i}{t}$. At each round, the $i_{th}$ distribution samples
$\limit{i}{t} = V$ with probability $\left(\half + \eps\right)$ and
$\limit{j}{i} = V$ with probability $\half$ for $j \ne i$. We now
mimic the proof of Theorem~5.1 in \citet{auer2003exp3}.

%% We use the same problem setup as Section~\ref{sec:lbcont}. There are $K$ reward
%% distributions.  The $i_{th}$ distribution samples $\limit{i}{t} = V$ with
%% probability $\left(\half + \eps\right)$ and $\limit{j}{i} = V$ with
%% probability $\half$ for $j \ne i$. However, this time we wish to
%% exploit the \emph{bandit} component of our problem. The fact that
%% we don't get to distinguish between $\limit{i}{t} = \fl{i}{t}$ and
%% $\limit{i}{t} > \fl{i}{t}$ when we round down should intuitively help
%% us show a stronger lower bound. It turns out that a slightly more
%% careful analysis mimicking the proof of the lower bound for $K$-armed
%% bandits in \citep{auer2003exp3} allows us to do this.
We
start with a lemma analogous to Lemma A.1 of Auer et
al~\citep{auer2003exp3}. Let $V_i = \sum_t\alloc{i}{t}$. Let $\E_i$ and
$\E_{\unif}$ denote expectations wrt the $i_{th}$ distribution and
uniform reward distribution respectively.
\begin{lemma}
  Let $f$ be a function of the reward sequence $\rew$ taking values in
  $[0,M]$. Then 
  \begin{small}
  $$\E_i f(\rew) \leq \E_{\unif}f(\rew) +
  M\sqrt{2\E_{\unif}[V_i]\ln\left(\frac{1}{1-4\eps^2}\right)}.$$ 
  \end{small}
  \label{lemma:allocbound}
\end{lemma}
\begin{proof}
  It is clear from H\"older's inequality and Pinsker's inequality that 
  $$\E_i[f(r)] - \E_\unif[f(r)] \leq M\|\P_i - \P_\unif\|_1 \leq
  M\sqrt{2\KL(\P_\unif||\P_i)}.$$
  Now we can proceed as in the proof of \cite{auer2003exp3}
  \begin{align*}
    \KL(\P_\unif||\P_i) &=
    \sum_{t=1}^T\KL(\P_\unif(r_t|\rew_{t-1})||\P_i(r_t||\rew_{t-1})\\
    &= \sum_{t=1}^T\left[ \sum_{j=1,j\ne i}^K\P_\unif(\alloc{j}{t} >
      0)\KL\left(\half||\half\right) + \P_\unif(\alloc{i}{t} >
      0)\KL\left(\half + \eps||\half\right)\right]\\
    &= \sum_{t=1}^T\P_\unif(\alloc{i}{t} > 0)\KL\left(\half +
    \eps||\half\right).
  \end{align*}
  As $\alloc{i}{t}$ is integer valued, $\P_\unif(\alloc{i}{t} > 0)
  \leq \E_\unif[\alloc{i}{t}]$. Hence
  \begin{align*}
    \KL(\P_\unif||\P_i) &\leq \sum_{t=1}^T\E_\unif[\alloc{i}{t}]\ln
    \left(\frac{1}{1-4\eps^2}\right)\\
    &= \E_\unif[V_i]\ln\left(\frac{1}{1-4\eps^2}\right).
  \end{align*}
\end{proof}

Using this lemma, we can prove a lower bound on the regret of any
algorithm that plays integer valued allocations.
\begin{theorem}
  Any algorithm that plays integer valued
  allocations has expected regret that is $\Omega\left(\sqrt{TV(K+V\ln
    K)}\right).$ 
  \label{thm:regretlbexp}
\end{theorem}
\begin{proof}
  The net reward of the algorithm when distribution $i$ is picked is
  given by 
  \begin{align*}
    &\E_i\sum_{t=1}^T\left[\sum_{j=1,j\ne i}^K\half\E_i\alloc{j}{t} +
      \left(\half + \eps\right)\E_i\alloc{i}{t}\right]\\
    &= \sum_{t=1}^T\left[\half(V - \E_i\alloc{i}{t}) + \left(\half +
      \eps\right) \E_i\alloc{i}{t}\right]\\
      &= \frac{TV}{2} + \eps\sum_{t=1}^T\E_i\alloc{i}{t} \\
    &= \frac{TV}{2} + \eps\E_i[V_i].
  \end{align*}
  As in the proof of Theorem~5.1 of \cite{auer2003exp3}, we now apply
  Lemma~\ref{lemma:allocbound} to the function $V_i$ of the reward
  sequence. As $V_i \in [0,TV]$, we get
  \begin{align*}
    \E_i[V_i] &\leq \E_\unif[V_i] +
    TV\sqrt{2\E_\unif[V_i]\ln\left(\frac{1}{1-4\eps^2}\right)}\\ 
    &\leq\E_\unif[V_i] + 2TV\eps\sqrt{\E_\unif[V_i]}.
  \end{align*}
  Then 
  $$\sum_{i-1}^K\E_i[V_i] \leq \sum_{i=1}^K\E_\unif[V_i] +
  2TV\eps\sum_{i=1}^K\sqrt{\E_\unif[V_i]}.$$
  Now $\sum_{i=1}^K\E_\unif[V_i] = TV$. Applying Jensen's inequality
  to the second term we get
  $$\sum_{i=1}^K\E_i[V_i] \leq TV + 2TV\eps\sqrt{KTV}.$$
  
  As the index $i$ was chosen uniformly at random, averaging over this
  choice gives an expected bound on the reward of
  $$\frac{1}{K}\sum_{i=1}^K\E_i[V_i] \leq \frac{TV}{K} +
  2TV\eps\sqrt{\frac{TV}{K}}.$$ 
  Noting again that the reward of optimal comparator is still
  $\left(\half+\eps \right)TV$, we get that the expected regret is
  $$\Omega\left(\eps\left(TV - \frac{TV}{K} +
  2TV\eps\sqrt{\frac{TV}{K}}\right)\right).$$ 
%%   Using arguments similar to \citet{auer2003exp3}, it is
%%   easy to show that the expected regret is lower bounded by
%%   $\Omega(\sqrt{TVK})$. 
  Setting $\eps$ optimally to $c\sqrt{\frac{K}{TV}}$ gives an
  $\Omega(\sqrt{TVK})$ lower bound.
  We also note that the lower bound
  of $\Omega(V\sqrt{T\ln K})$ shown for continuous-valued allocations
  applies to the integer-valued case as well. Combining the
  two, we get that the regret is
  \begin{small}
  \begin{align*}
    \Omega(\max\{\sqrt{TVK},V\sqrt{T\ln K}\}) 
    = \Omega\left(\sqrt{T}\left(\sqrt{VK} + V\sqrt{\ln
    K}\right)\right).
  \end{align*}
\end{small}
\end{proof}

There is a gap between our lower and upper bounds in this case.
We do not know which bound is loose.

\subsection{Efficient sampling for integral allocations} 
\label{sec:greedy}
All that remains to specify in Algorithm~\ref{alg:intallocexp} is the
construction of the distribution $p$ over subsets at every
round. Since we don't know what the distribution is, we cannot sample
from it easily it would seem. If $K$ is small, one can use
non-negative least squares to find the distribution that has the
given marginals. However, once the number of venues $K$ is large, $p$
is a distribution over $\binom{K}{m}$ subsets, for which the least
squares solver might be too slow. One way around is to use the idea of
greedy approximations in Hilbert Spaces from the classic paper of
\citep{jones92greedy}. We can greedily construct a distribution on
subsets which matches the marginals on every element approximately in
an efficient manner. Exact sampling from the distribution without ever
constructing it explicitly is also possible. The explicit algorithms
giving the implementations can be found in the full version of the
paper.

\section{Experimental results}
\label{sec:expt}

We compared four methods experimentally. We refer to
Algorithms~\ref{alg:fracalloc} and~\ref{alg:intallocexp} as \expgrad~ 
and \expthree~ respectively. We also run the Optimistic Kaplan Meier
estimator based algorithm of \citep{gknv2009darkpools},
which is called \optkm. Finally we implemented the parametric maximum
likelihood estimation-allocation based algorithm described in
\citep{gknv2009darkpools} as well, which we call \parml. As we did not
have access to real dark pool data, we decided to implement a data
simulator similar to \citep{gknv2009darkpools}. We used a combination
of a \emph{Zero Bin} parameter and power law distribution to generate
the $\limit{i}{t}$'s while the sequence $\vol{t}$ was kept
fixed. Parameters for the Zero Bin and power law were set to lie in the
same regimes as the ones observed in the real data of
\citep{gknv2009darkpools}.

We started by generating the data from the parametric model of
\citep{gknv2009darkpools}. We used 48 venues, $T=2000$ to match the
experiments of \citep{gknv2009darkpools}. The values of
$\limit{t}{i}$'s were sampled iid from Zero Bin+Power law
distributions with appropriately chosen parameters. A plot of the
resulting cumulative rewards averaged over 100 trial runs can be seen
in Figure~\ref{fig:cumulative48venue}.
\begin{figure}
  \centering
  \includegraphics[scale=.30]{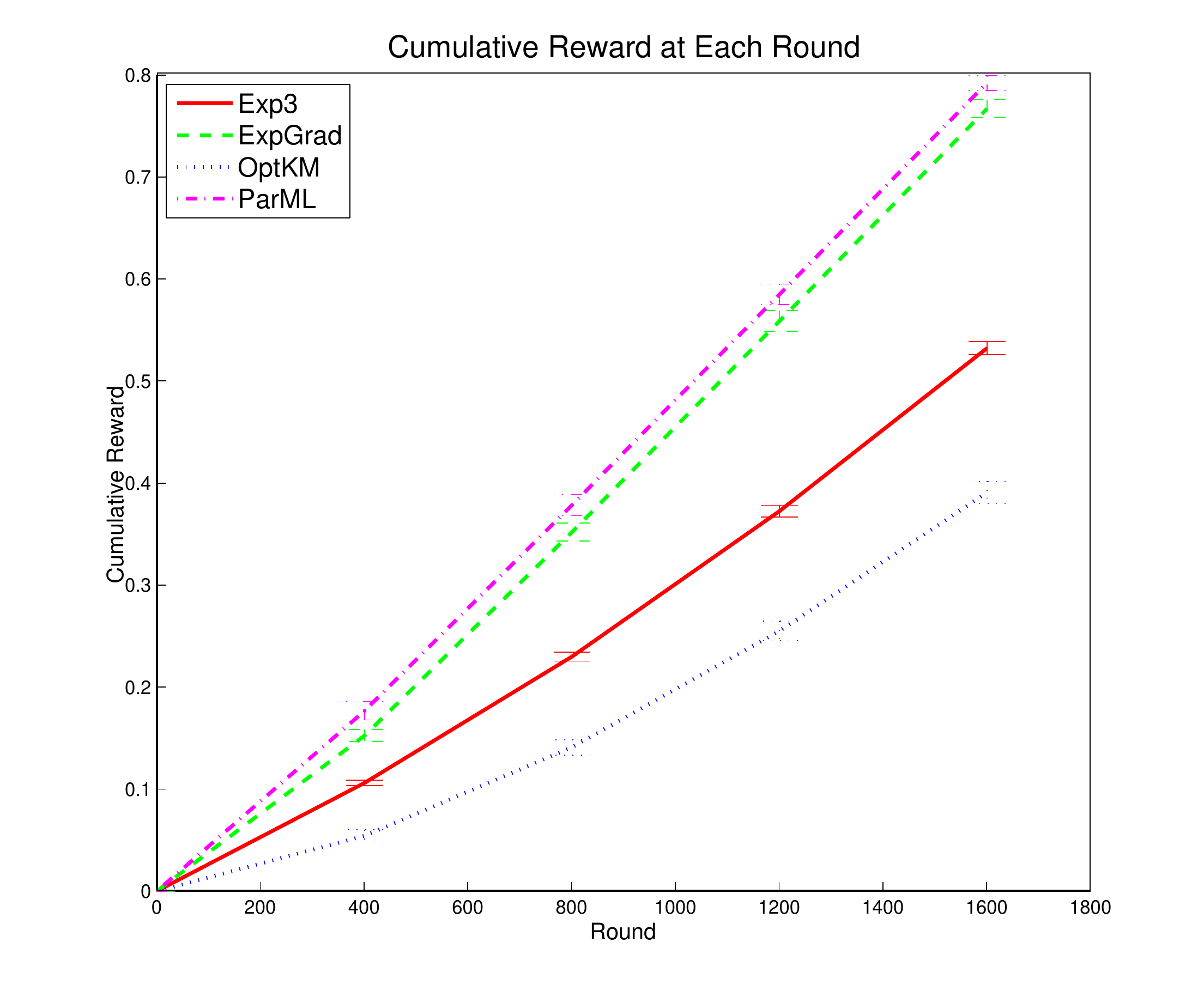} 
  \caption{Cumulative rewards for each algorithm as a function of the
    number of rounds when run on the parametric model of
    \citep{gknv2009darkpools} averaged over 100 trials} 
  \label{fig:cumulative48venue}
\end{figure}

We see that \parml~ has a slightly superior performance on this data,
understandably as the data is being generated from the specific
parametric model that the algorithm is designed for. However,
\expgrad~ gets net allocations quite close to \parml. Furthermore,
both \expthree~ and \expgrad~ are far superior to the performance \optkm~
which is our true competitor in some sense being a non-parametric
approach just like ours. 

%% It is clear that \par~has the best performance on this data, even
%% though it comes from the parametric model of \parml. Furthermore, the
%% allocations of \expthree~ are quite close to \parml~ as well, and seem
%% to catch up over time. It is encouraging to note that the
%% performance of our algorithms is competitive with \parml, even under
%% the parametric model. \optkm~has a highly inferior performance, perhaps
%% due to the overfitting observed in~\citep{gknv2009darkpools}.

Next, we study the performance of all four algorithms under a variety
of adversarial scenarios. We start with a simple setup of two
venues. The parameters of the power law initially favor Venue 1 for 12500
rounds, and then we switch the power law parameters to favor Venue
2. We study both the cumulative rewards as well as the allocations to
both venues for each algorithm. Clearly an algorithm will be more
robust to adversarial perturbations if it can detect this change
quickly and switch its allocations accordingly. We show the results of
this experiment in Figure~\ref{fig:2venue}.
\begin{figure}
  \centering
  \begin{tabular}{cc}
      \includegraphics[width=.225\textwidth]{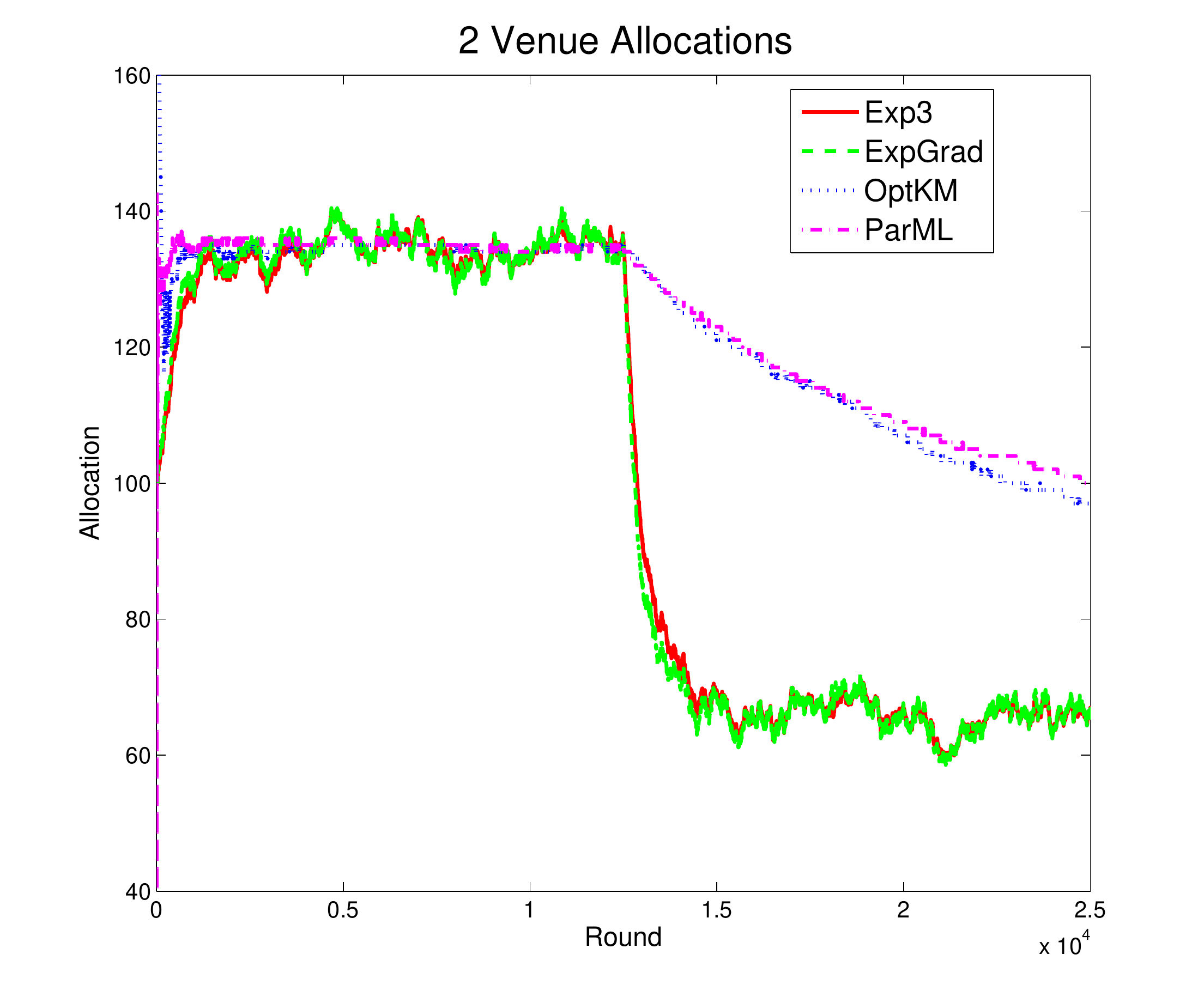}
    &
      \includegraphics[width=.225\textwidth,height=90pt]{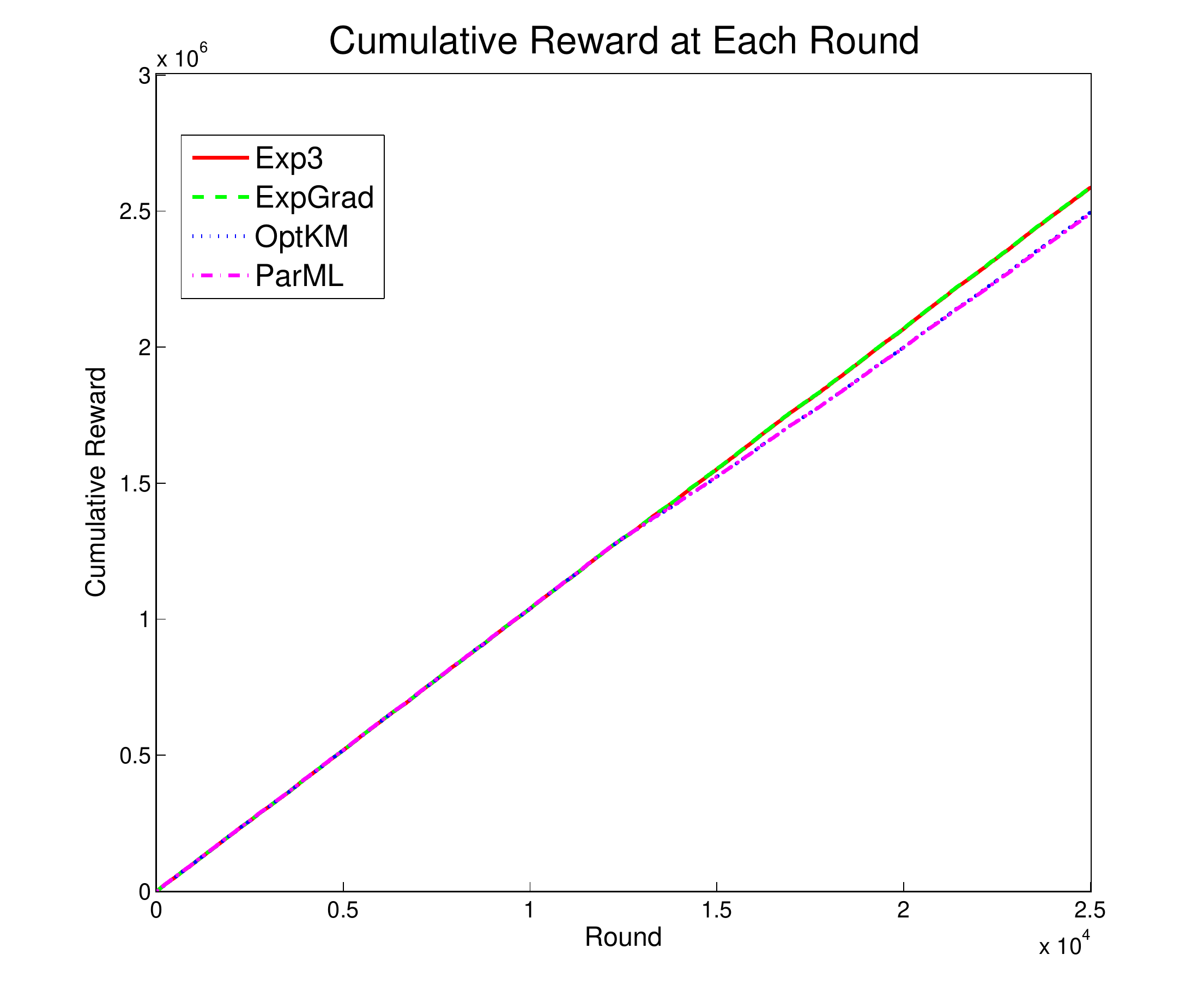}\\
      (a)&(b)
  \end{tabular}
  \caption{Allocations to the 2 venues and cumulative rewards for the
    different algorithms. Note the inability of \parml~and \optkm~to
    effectively switch between venues when distributions
    switch. \expgrad~ and \expthree~also achieve higher cumulative
    rewards. } 
  \label{fig:2venue}
\end{figure}

Because of just 2 venues, rounding has a rather negligible effect in
this case and both our methodshave an almost identical performance. Our
algorithms \expgrad~ and \expthree~ switch much faster to the new 
optimal venue when distributions switch. Consequently, the cumulative
reward of both our algorithms also turns out significantly higher as
shown in Figure~\ref{fig:2venue}(b).

We wanted to investigate how this behavior changes when the switching involves
a larger number of venues. We created another experiment where there
are 5 venues, maximum volume $V=200$. Venues 1 and 5 oscillate between
getting very favorable and unfavorable $\beta$ values ($\beta$ is the power
law exponent). Other venues also switch, but between less extreme
values. Allocations to all 5 venues for each algorithm are shown in
Figure~\ref{fig:5venue}. 
\begin{figure}
  \centering
  \begin{tabular}{cc}
      \includegraphics[width=.225\textwidth]{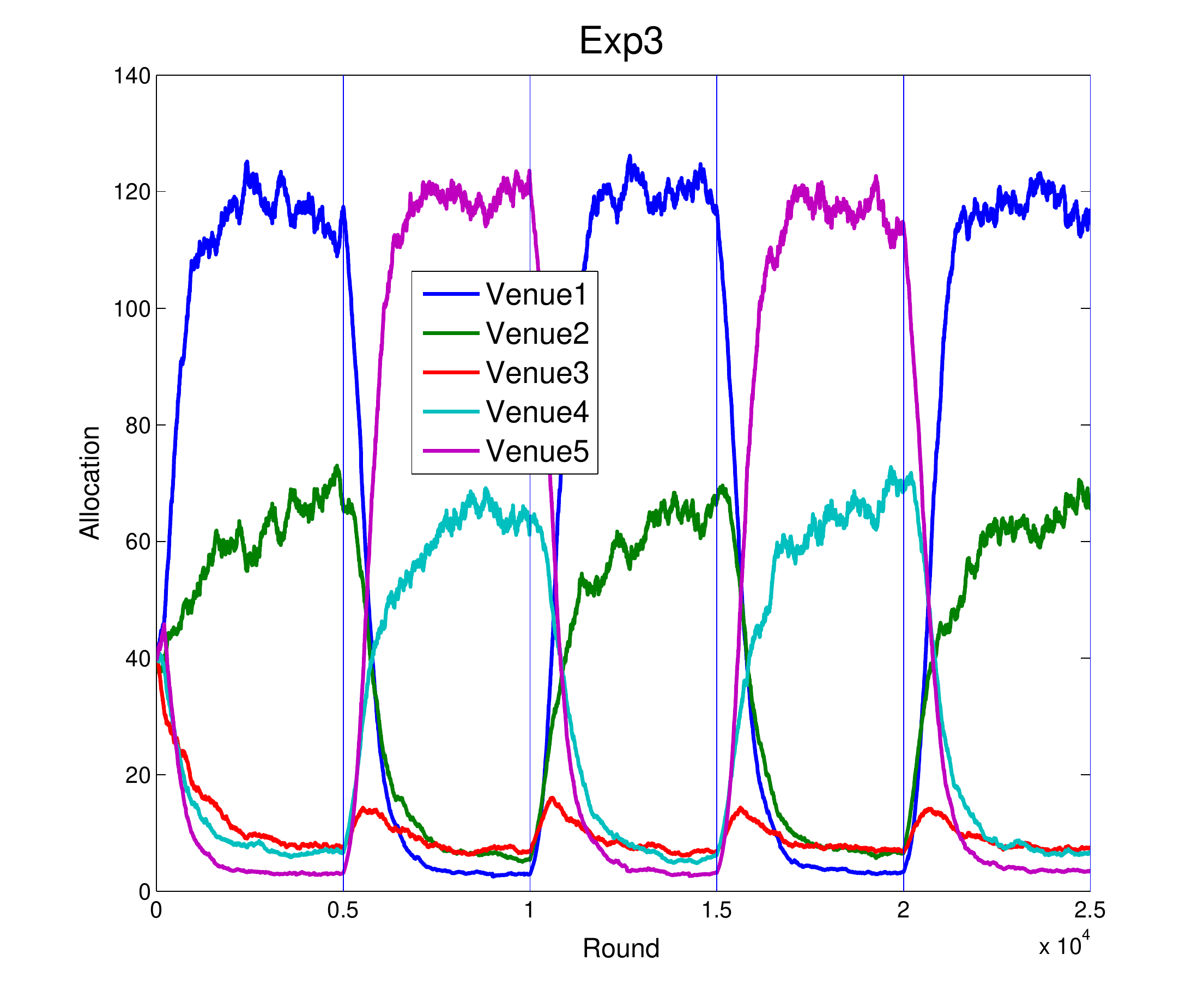}
    &
      \includegraphics[width=.225\textwidth]{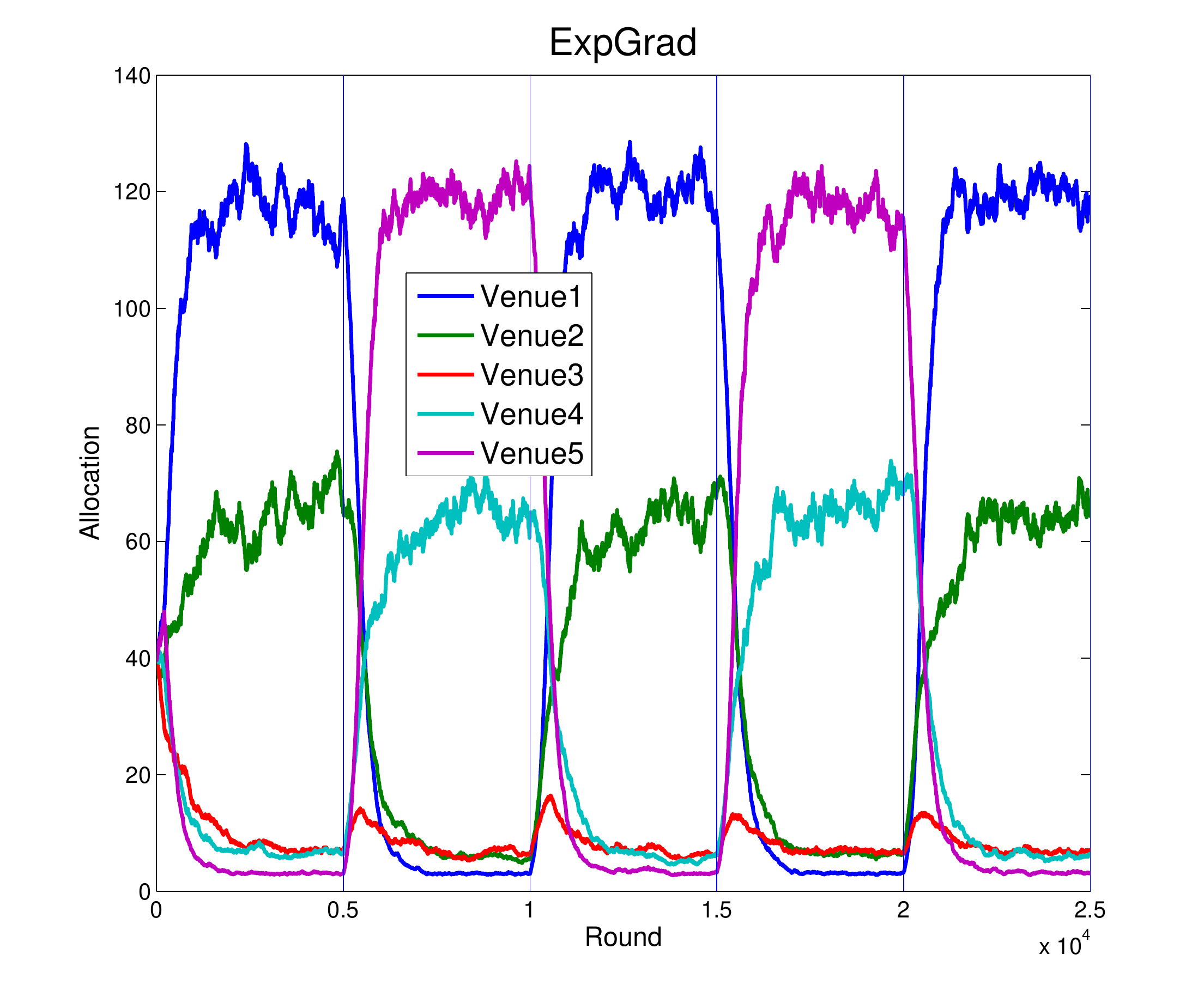}\\
      (a)&(b)\\
      \includegraphics[width=.225\textwidth]{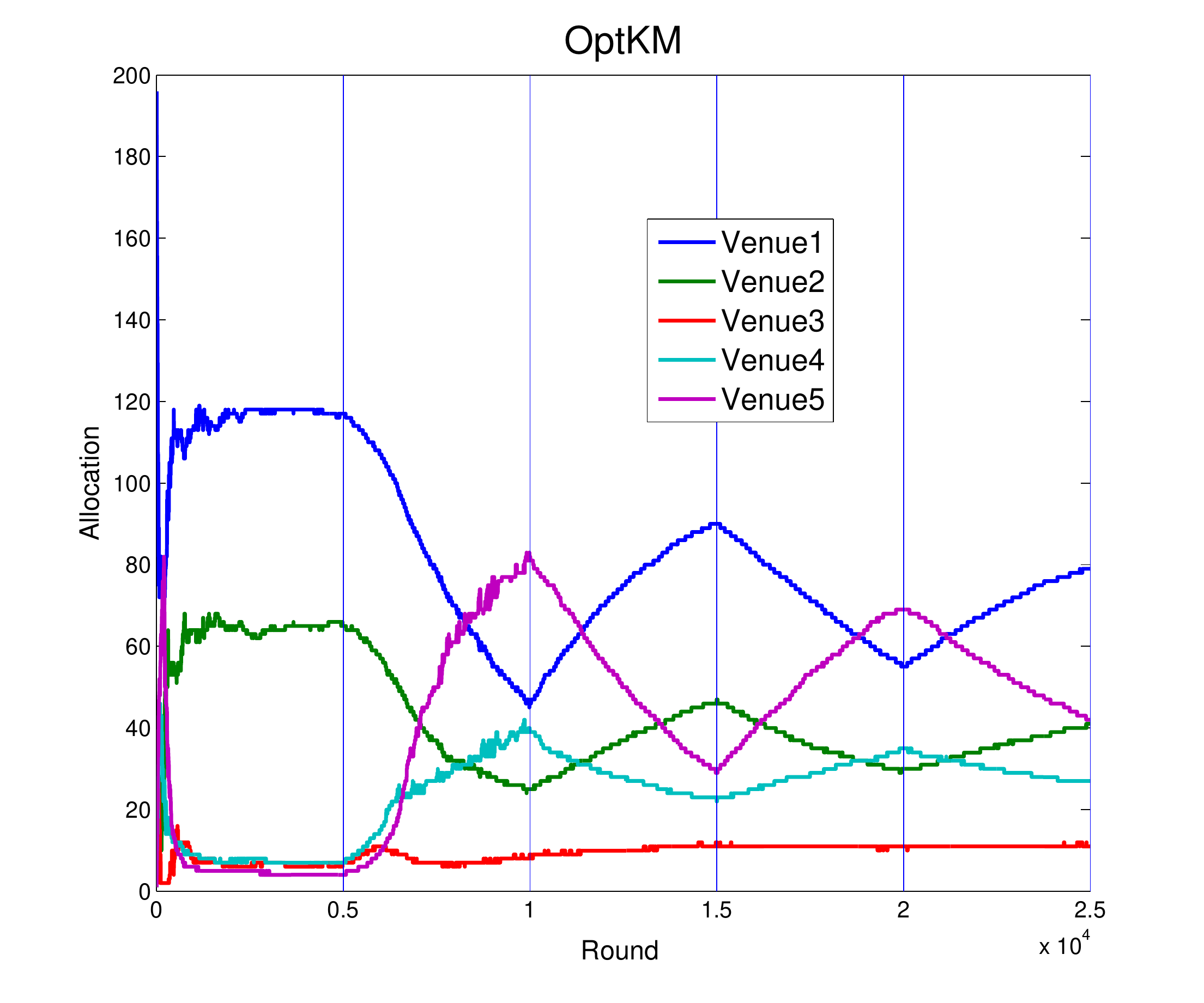}
    &
      \includegraphics[width=.225\textwidth]{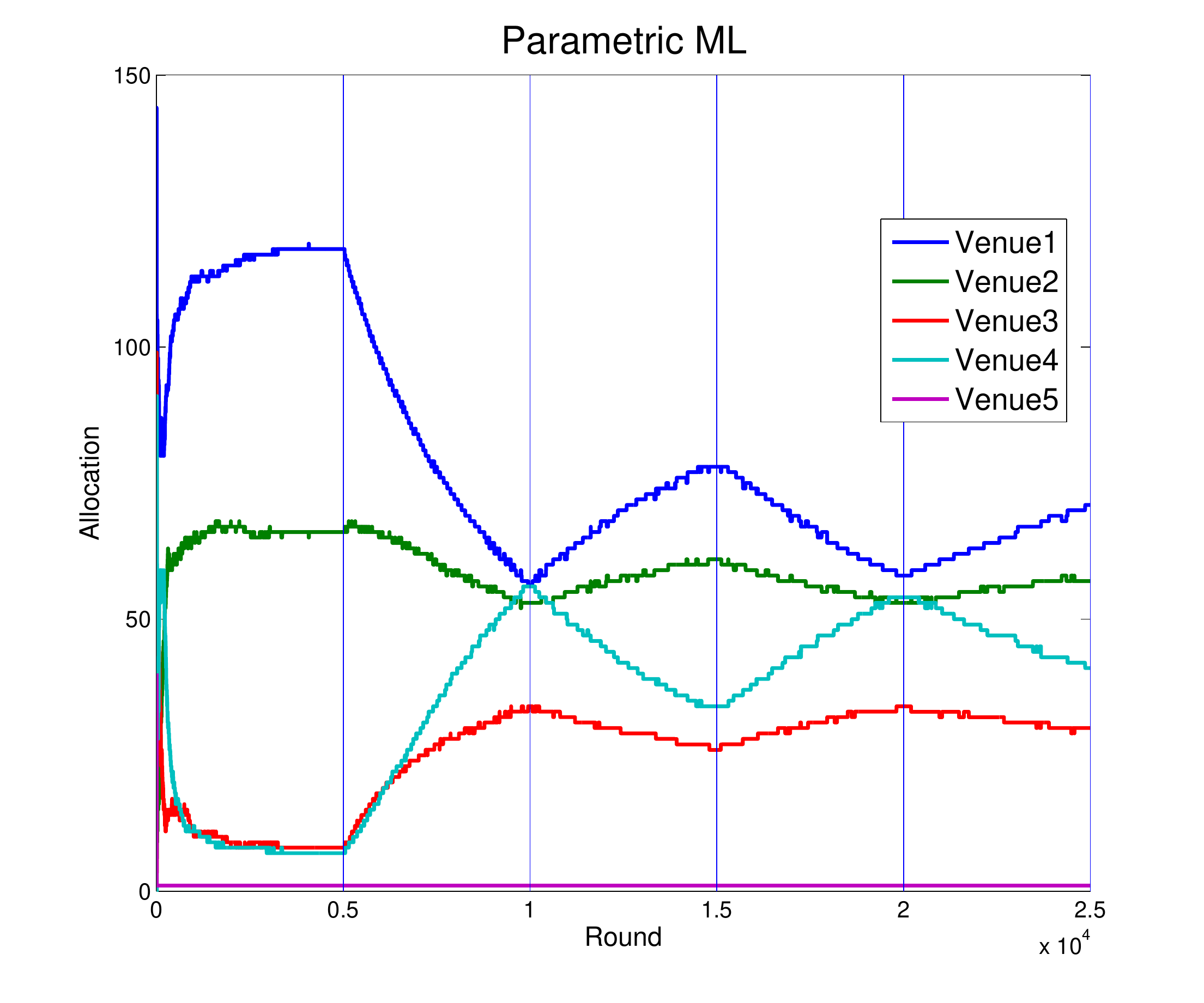}\\
      (c)&(d)
  \end{tabular}
  \caption{Allocations to the 5 venues for the
    different algorithms. Note the poor switching of \optkm~between
    venues when distributions switch. \parml~completely fails
    on this problem. \expthree~and \expgrad~correctly identify both long
    and short range trends (see text).} 
  \label{fig:5venue}
\end{figure}

Once again both \expthree~and \expgrad~identify both the long
range trend (favorability of venues 1, 5 over the others) and short range
trend (favoring venue 1 over 5 in certain phases). There is a gap
between \expthree and \expgrad~this time, however, as rounding does
start to play a role with 5 venues. \optkm~adapts
somewhat, although it still doesn't reach as high an allocation level
as \expthree~after switching to a new venue. \parml~completely fails
to identify this switching. We also studied the behavior of algorithms as $V$ is
scaled on the same problem. Figure~\ref{fig:5venuecum} plots the
cumulative reward of each algorithm for $V=200$ and $V=400$. It is
clear that \expgrad~and \expthree~still comprehensively outperform
others. 

\begin{figure}
  \centering
  \begin{tabular}{cc}
      \includegraphics[width=.25\textwidth]{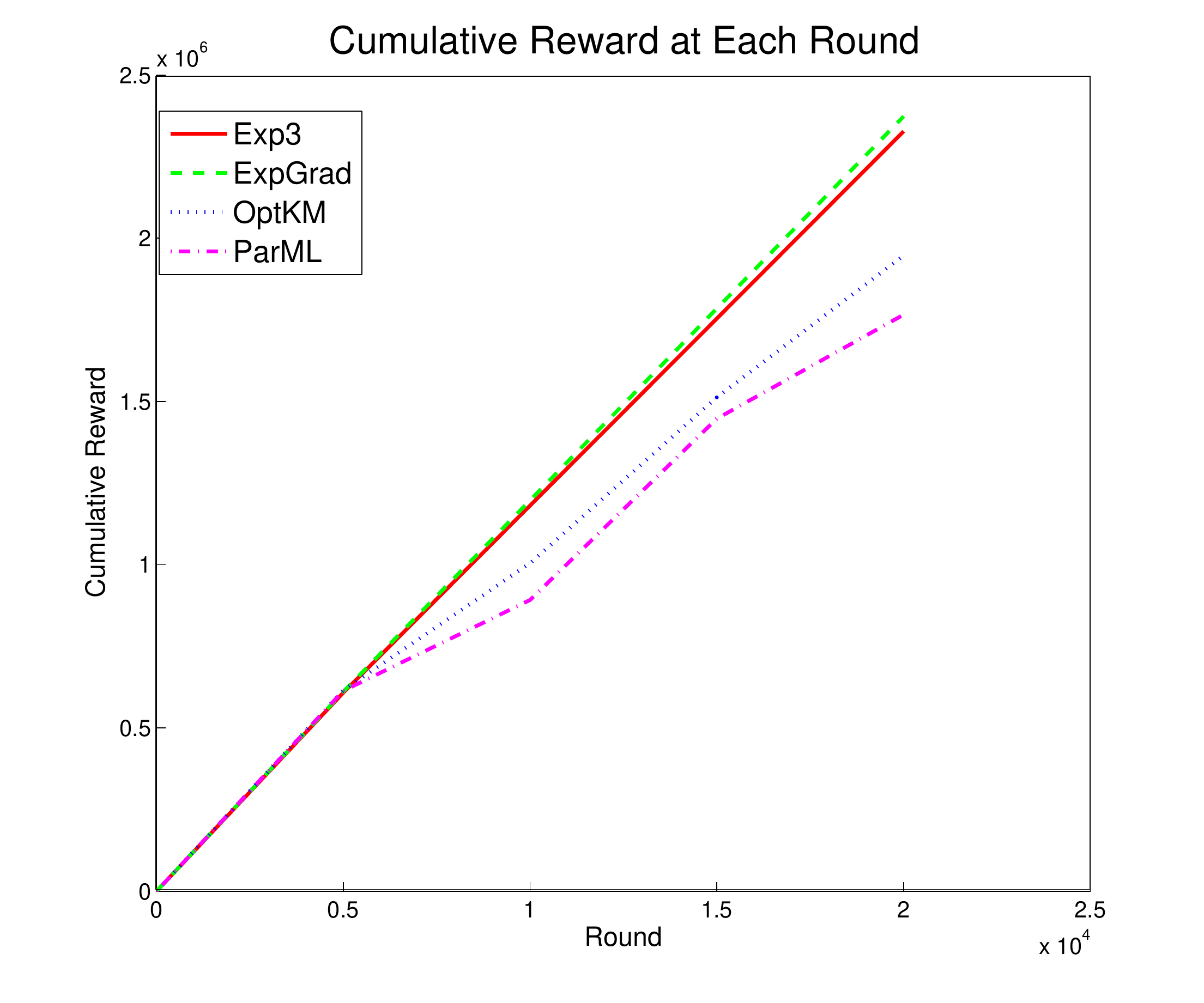}
    &
      \includegraphics[width=.25\textwidth]{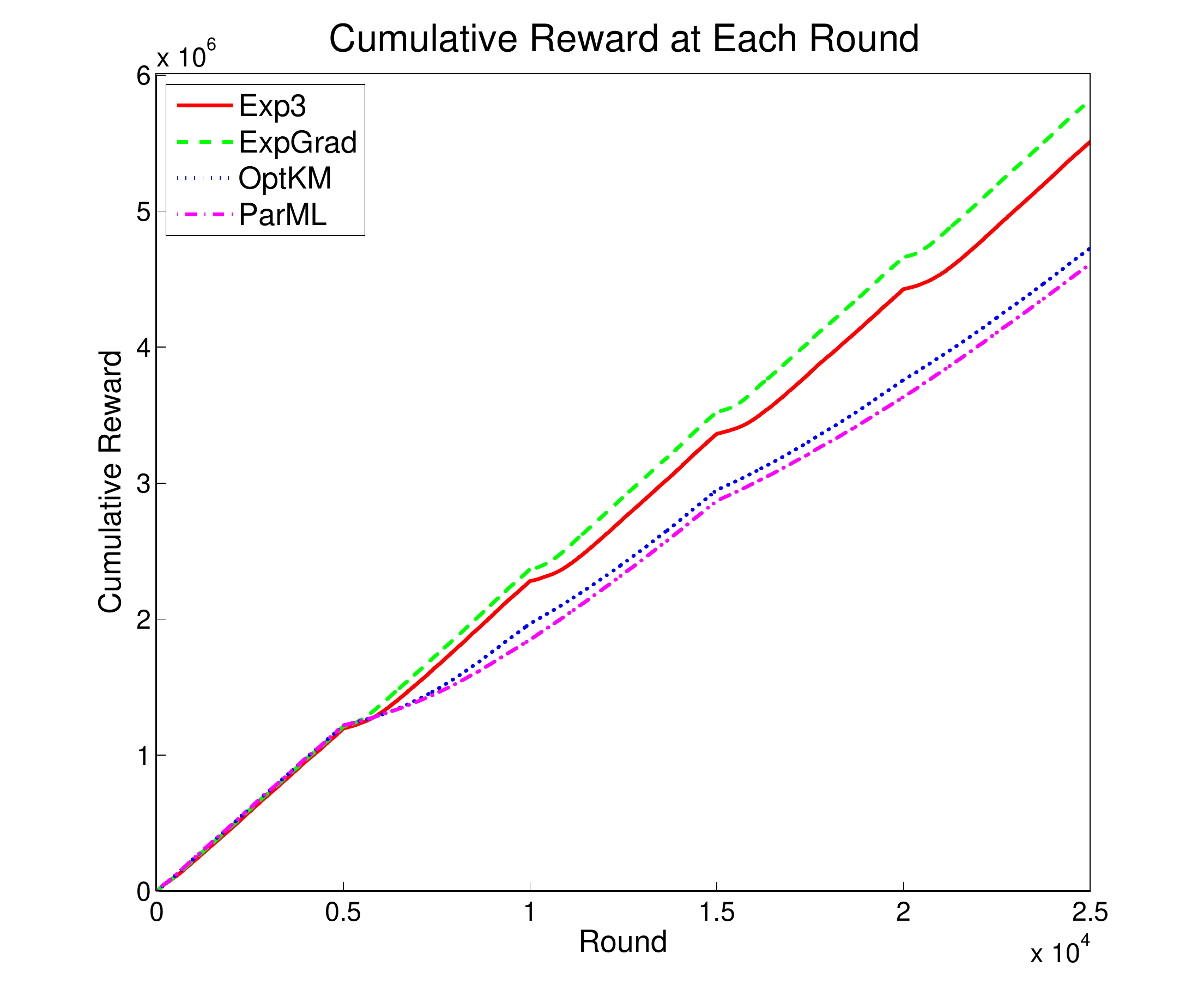}\\
      (a)&(b)
  \end{tabular}
  \caption{Cumulative rewards for each algorithm when distributions
    switch between 5 venues, for $V=200$(left) and $V=400$. Note the
    superior performance of \expgrad~and\expthree.} 
  \label{fig:5venuecum}
\end{figure}

In summary, it seems that our algorithms are competitive with those
of \citep{gknv2009darkpools} when the data is drawn from
their parametric model. When their assumptions about iid data are
not satisfied, we significantly outperform those algorithms. We note that we
have only experimented with oblivious adversaries here. The gulf in
performance may be even wider for adaptive adversaries.

\bibliographystyle{apalike}
\begin{small}
\bibliography{darkpools}
\end{small}
\end{document}